\documentclass{article}

\usepackage{microtype}
\usepackage{graphicx}
\usepackage{subcaption}
\usepackage{booktabs} 

\usepackage{hyperref}

\usepackage[preprint]{icml2026}

\usepackage[utf8]{inputenc} 
\usepackage[T1]{fontenc}    
\usepackage{hyperref}       
\usepackage{url}            
\usepackage{booktabs}       
\usepackage{amsfonts} 
\usepackage{tcolorbox}
\usepackage{nicefrac}       
\usepackage{microtype}      
\usepackage{xcolor}         
\usepackage{amsmath}
\usepackage{amssymb}
\usepackage{mathtools}
\usepackage{tabularray}
\usepackage{amsthm}

\newcommand{\ie}{\textit{i.e.}}
\newcommand{\eg}{\textit{e.g.}}

\usepackage[capitalize,noabbrev]{cleveref}

\theoremstyle{plain}
\newtheorem{theorem}{Theorem}[section]

\newtheorem{lemma}[theorem]{Lemma}

\theoremstyle{definition}
\newtheorem{definition}[theorem]{Definition}

\theoremstyle{remark}

\usepackage[textsize=tiny]{todonotes}

\icmltitlerunning{Curiosity is Knowledge}

\begin{document}

\twocolumn[
  \icmltitle{Curiosity is Knowledge: \\ Self-Consistent Learning and No-Regret Optimization with Active Inference}



  \icmlsetsymbol{equal}{*}

  \begin{icmlauthorlist}
    \icmlauthor{Yingke Li}{mit}
    \icmlauthor{Anjali Parashar}{mit}
    \icmlauthor{Enlu Zhou}{gatech}
    \icmlauthor{Chuchu Fan}{mit}
  \end{icmlauthorlist}

  \icmlaffiliation{mit}{Department of Aeronautics and Astronautics,  Massachusetts Institute of Technology, Cambridge, MA 02139, USA }
  \icmlaffiliation{gatech}{School of Industrial and Systems Engineering, Georgia Institute of Technology, Atlanta, GA 30332, USA}

  \icmlcorrespondingauthor{Chuchu Fan}{chuchu@mit.edu}

  \icmlkeywords{Active inference, Bayesian consistency, Regret bound}

  \vskip 0.3in
]



\printAffiliationsAndNotice{}  

\begin{abstract}
Active inference (AIF) unifies exploration and exploitation by minimizing the Expected Free Energy (EFE), balancing epistemic value (information gain) and pragmatic value (task performance) through a curiosity coefficient. Yet it has been unclear when this balance yields both coherent learning and efficient decision-making: insufficient curiosity can drive myopic exploitation and prevent uncertainty resolution, while excessive curiosity can induce unnecessary exploration and regret. We establish the first theoretical guarantee for EFE-minimizing agents, showing that a single requirement—sufficient curiosity—simultaneously ensures self-consistent learning (Bayesian posterior consistency) and no-regret optimization (bounded cumulative regret). Our analysis characterizes how this mechanism depends on initial uncertainty, identifiability, and objective alignment, thereby connecting AIF to classical Bayesian experimental design and Bayesian optimization within one theoretical framework. We further translate these theories into practical design guidelines for tuning the epistemic–pragmatic trade-off in hybrid learning-optimization problems, validated through real-world experiments.
\end{abstract}

\section{Introduction}

A fundamental challenge in sequential decision-making under uncertainty is balancing \textit{exploration} and \textit{exploitation}: an agent must gather information to reduce uncertainty while simultaneously pursuing task objectives. 
This tension leaves a primary challenge in classical Bayesian formulations. Bayesian Optimization (BO) is primarily goal-seeking: it selects actions to maximize reward (or minimize cost) under uncertainty, typically emphasizing exploitation tempered by uncertainty-driven exploration \citep{Shahriari2016TakingOptimization, Frazier2018AOptimization}. In contrast, Bayesian Experimental Design (BED) is information-seeking: it selects experiments to maximize information gain about unknown parameters, often without an explicit performance objective \citep{Rainforth2023ModernDesign}. However, real systems rarely permit a clean separation between these two modes; instead, \textit{seeking knowledge} and \textit{achieving goals} are often deeply intertwined objectives. As a sequel, a principled agent should unify goal-seeking and information-seeking within a single decision rule.

Recently, active inference (AIF) \citep{Friston2010TheTheory, Friston2017ActiveTheory} has arisen as a promising way to offer such a unification. By minimizing the Expected Free Energy (EFE), AIF casts action selection as a single variational objective that naturally decomposes into an epistemic term (expected information gain) and a pragmatic term (expected regret). These two terms are balanced by a coefficient called \textit{curiosity}, which sets the trade-off between \textit{learning} and \textit{optimization}.
Despite the empirical success of this new paradigm \citep{AIFBO-application-arxiv}, the formal role of curiosity in guaranteeing coherent learning and efficient optimization remains incomplete.
If curiosity is too low, the agent can become myopic, prematurely exploiting, and failing to resolve uncertainty. On the contrary, if curiosity is too high, the agent may be devoted to unnecessary exploration, discarding the goals, and incurring catastrophic regret.

This raises a critical dilemma in AIF that we formally addressed in this paper:
\begin{quote}
    When can minimizing EFE guarantee both \textit{self-consistent learning} (\ie, posterior convergences to the truth) and \textit{no-regret optimization} (\ie, with bounded cumulative regret)?
\end{quote}

We establish the \textbf{first theoretical guarantee} of \textit{posterior consistency} and \textit{bounded cumulative regret} for agents minimizing EFE. 
We prove that these two fundamental properties are not merely compatible but are jointly ensured by a single, shared mechanism: \textit{sufficient curiosity}.
That is, we derive a lower bound of the curiosity coefficient that can prevent the epistemic value from being dominated by the pragmatic term, and prove that EFE minimization can achieve both self-consistent learning and no-regret optimization simultaneously when the curiosity coefficient is larger than this lower bound. Specifically, our contributions are threefold:

First, we prove that \textbf{posterior consistency} is guaranteed under the sufficient curiosity condition through Theorem~\ref{thm:self-consistency}. This theorem provides a sample-complexity bound that explicitly depends on \textit{prior entropy}, \textit{model discriminability}, and \textit{curiosity}. 

Second, we establish a \textbf{general regret bound} for AIF with Gaussian Process (GP) settings in Theorem~\ref{thm:URB}. This bound holds for generalized heuristic estimators and regret functions, formally linking convergence rates to \textit{smoothness}, \textit{heuristic alignment}, and the same \textit{sufficient curiosity} condition. In particular, the classical BO-style regret analysis is a special case of our general regret bound.

Third, we translate these theories into \textbf{practical design guidelines}. Using both synthetic and real-world problems, we provide concrete prescriptions for the epistemic-pragmatic balance, hyperparameter selection, and objective function design for empirical applications.

Notably, we show that the requirements to ensure self-consistent learning and no-regret optimization depend on the same sufficient curiosity condition. This finding elevates curiosity from an ad-hoc exploration heuristic into an intrinsic regularizer that couples belief updating and decision-making. Therefore, this work reveals an essential fact that, a sufficient curiosity is the critical mechanism enabling simultaneous self-consistent learning and no-regret optimization, in other words, \textbf{curiosity is knowledge}.

\section{Preliminaries}

\subsection{Bayesian Optimization}
Given an unknown objective function $f:\mathcal{X} \mapsto \Re$, BO seeks to identify the input $x^{\ast}$ that maximizes the objective $f$ over an admissible set of queries $\mathcal{X}$, \ie, 
$
    x^{\ast} = \arg \max_{x \in \mathcal{X}} f(x).
$
To achieve this goal, BO relies on a \textit{surrogate model} that provides a probabilistic representation of the objective $f$, and uses this information to compute an \textit{acquisition function} to drive the selection of the most promising sample to query.

\textbf{Surrogate model.} We assume the available information regarding the objective function $f$ be stored in the dataset $\mathcal{D}_{t}:=\{(x_{1},y_{1}), \dots, (x_{t},y_{t}))\}$, where $y_{t} \sim \mathcal{N}(f(x_{t}),\sigma^{2}(x_{t}))$ is the noisy observation of the objective function by assuming the noise follows a zero-mean normal distribution with a standard deviation $\sigma$. 
The surrogate model depicts possible explanations of $f$ as $f(x) \sim p(f(x)|\mathcal{D}_{t})$ applying a joint distribution over its behavior at each sample $x \in \mathcal{X}$.
In Bayesian inference, the prior distribution of the objective $p(f(x))$ is combined with the likelihood function $p(\mathcal{D}_{t}|f(x))$ to compute the posterior distribution $p(f(x)|\mathcal{D}_{t})\propto p(\mathcal{D}_{t}|f(x))p(f(x))$, representing the updated beliefs about $f(x)$. 
Typically, Gaussian processes (GPs) have been widely used as the surrogate model for BO due to their efficient posterior sampling that enables cheap, gradient-based optimization of the acquisition function to propose new query points.
GP is specified by a joint normal distribution $p(f(x)|\mathcal{D}_{t})=\mathcal{N}(\mu_{t}(x),\kappa_{t}(x, x'))$ with mean $\mu_{t}(x)$ and kernel function $\kappa_{t}(x, x')$, where $\mu_{t}(x)$ represents the prediction and $\kappa_{t}(x, x')$ the associated uncertainty.

\textbf{Acquisition function.} The surrogate model is utilized to decide the next sample $x_{t+1}\in \mathcal{X}$ through the maximization of an acquisition function $\alpha:\mathcal{X} \mapsto \Re$, \ie, $x_{t+1}=\arg \max_{x\in \mathcal{X}}\alpha(x|\mathcal{D}_{t})$, where $\alpha(x|\mathcal{D}_{t})$ provides a measure of the improvement that the next query is likely to provide with respect to (w.r.t.) the current surrogate model of the objective function. 
Many acquisition functions have been proposed, including \textit{probability of improvement} \citep{Mockus1975OnExtremum}, \textit{expected improvement} \citep{Jones1998EfficientFunctions}, \textit{upper confidence bound}, and various \textit{entropy search} methods \citep{Hennig2012EntropyOptimization, Hernandez-Lobato2014PredictiveFunctions, Wang2017Max-valueOptimization, Hvarfner2022JointOptimization, Neiswanger2021BayesianInformation},
as well as practical approaches to optimize them \citep{Wilson2018MaximizingOptimization}.

\subsection{Bayesian Experimental Design}

Rather than optimizing an objective function $f(x)$, the purpose of BED is to sequentially select a set of experimental designs $x \in \mathcal{X}$ and gather outcomes $y$, to maximize the amount of information obtained about certain \textit{parameters of interest}, denoted as $\theta$. 
The parameters $\theta$ can correspond to some explicit model parameters, or any implicitly defined quantity of interest (\eg, the optimum of a function, the output of an algorithm, or future downstream predictions).

Based on the current history of experiments $\mathcal{D}_{t}:=\{(x_{1},y_{1}), \dots, (x_{t},y_{t}))\}$, BED seeks to find the next experimental design $x_{t+1}$ by maximizing the \textit{expected information gain} (EIG) \citep{Chaloner1995BayesianReview} that a potential experimental outcome $y_{t+1}$ can provide about $\theta$, measured in terms of expected entropy reduction of the posterior distribution of $\theta$:
\begin{equation*}
\begin{split}
    \text{EIG}(x|\mathcal{D}_{t})&=H[p(\theta|\mathcal{D}_{t})] - \mathbb{E}_{p(y|x, \mathcal{D}_{t})}[H[p(\theta|\mathcal{D}_{t}\cup(x,y))]]\\
    &=I(\theta;(x,y)|\mathcal{D}_{t}),
\end{split}
\end{equation*}
where $H(\cdot)$ and $I(\cdot)$ denote the entropy and mutual information, respectively.

\section{Active Inference Bridges Learning and Optimization}

BO and BED emphasize two complementary behaviors in sequential decision-making under uncertainty. BO is predominantly \emph{goal-seeking}: it allocates queries to locate an optimizer with few evaluations. BED is predominantly \emph{information-seeking}: it selects experiments to maximize expected information gain about latent parameters. Although historically developed as separate paradigms, both can be viewed as different responses to the same imperative: act under uncertainty so as to improve outcomes while reducing uncertainty.

Active inference (AIF) \citep{Friston2010TheTheory, Friston2017ActiveTheory} provides a principled unification through the Expected Free Energy (EFE). Minimizing EFE yields a single decision objective that decomposes into two components: a \emph{pragmatic} term that favors actions expected to produce preferred outcomes (exploitation), and an \emph{epistemic} term that favors actions expected to reduce uncertainty about latent quantities (exploration). In this view, exploration and exploitation are not separate heuristics but two aspects of one principle.

Building on this perspective, \citet{AIFBO-application-arxiv} introduced \textbf{pragmatic curiosity}, which implements AIF as the acquisition rule
\begin{equation}\label{eqn:AIF}
    \alpha(x|\mathcal{D}_{t})=\beta_{t} I(s;(x,y)|\mathcal{D}_{t}) - \mathbb{E}_{p(y| x, \mathcal{D}_{t})}[h(y|\mathcal{D}_{t})],
\end{equation}
where $s \in \mathcal{S}$ denotes the latent \textit{parameter(s)} of interest, defining the \textit{learning space}; $h(y|\mathcal{D}_{t})$ is a problem-dependent \textit{potential energy function} quantifying the \textit{instantaneous regret} associated with outcome $y$; and $\beta_{t} \ge 0$ is the \textit{curiosity coefficient} that sets the exchange rate between information gain and pragmatic cost.

Equation~\eqref{eqn:AIF} makes the key design tension explicit: $\beta_t$ controls how strongly the epistemic drive influences action selection. Too little curiosity risks myopic exploitation and insufficient learning; too much curiosity risks over-exploration and degraded task performance. This motivates the central question that we aim to address in the following sections: \emph{how does curiosity govern performance, both in learning and in optimization?}

\section{Performance Evaluation in Learning and Optimization}

We investigate the performance of AIF principle in terms of the ultimate goals of (i) learning: learn the true model, characterized by the \textit{posterior consistency}; and (ii) optimization: reduce the regret, evaluated through \textit{bounded cumulative regret}. We begin by formalizing these notions.

\begin{definition}[\textbf{Posterior Consistency}]
Let $q_t(s)$ denote the posterior belief over the latent parameter $s \in \mathcal{S}$ after observing dataset $\mathcal{D}_{t}$. 
The learning process is said to be \textit{(strongly) consistent} if with probability one (\textit{w.p.1.}),
\begin{equation*}
    \lim_{t\rightarrow \infty} q_t(s) = \delta_{s^{\ast}},
\end{equation*}
where $\delta_{s^{\ast}}$ is the Dirac measure concentrated at the true parameter value $s^{*}$. 
\end{definition}

\begin{definition}[\textbf{Regret Function}]
Let $\mathcal{Y}$ denote the outcome space.
A regret function is a mapping $r:\mathcal{Y} \rightarrow \mathbb{R}_{\ge 0}$ that quantifies the deviation of an outcome $y$ from its desired state $y^{\ast}$. 
The smaller the value of $r(y)$, the closer $y$ is to the desired outcome. Specifically,
\begin{itemize}
\item When $y^{\ast}$ corresponds to the global optimum of an objective function, $r(y)$ reduces to the conventional notion of instantaneous regret in BO;
\item When all outcomes are equally desirable or informative, as in BED, $r(y)\equiv 0$;
\item Intermediate cases correspond to hybrid objectives that encode both performance and safety constraints.
\end{itemize}
\end{definition}

Since the true regret function $r(y)$ is generally unknown or intractable, we define a time-varying surrogate energy function that encodes the current model’s internal assessment of outcome desirability and is updated recursively as new data are acquired:
\begin{definition}[\textbf{Potential Energy Function}]
The potential energy function $h_t:\mathcal{Y} \rightarrow \mathbb{R}_{\ge 0}$ serves as a heuristic estimator of regret, conditioned on past observations:
\begin{equation*}
    h_{t}(y):=h(y|\mathcal{D}_{t-1}).
\end{equation*}
\end{definition}

With these definitions in place, we are now equipped to analyze the theoretical properties of AIF principle. 
In the following sections, we show that EFE minimization ensures \textit{posterior consistency} and \textit{bounded cumulative regret} under mild regularity assumptions. 
In particular, both results rely on a shared \textit{sufficient curiosity} condition: the curiosity coefficient $\beta_{t}$ must be sufficiently large so that the epistemic value of information gain is not overshadowed by the expected energy penalty. 

\section{Self-Consistent Learning}

We first derive the conditions under which AIF principle guarantees posterior consistency in learning. 
Specifically, we show that when the curiosity is sufficiently large to ensure continual information gain, the posterior belief converges to the true generative model.

\begin{theorem}[\textbf{Posterior Consistency in AIF}]\label{thm:self-consistency}
Let $s$ be a discrete latent parameter of the model with parameter space $\mathcal{S}$, and $s^{\ast} \in \mathcal{S}$ denote the true (data-generating) parameter. 
At each iteration $t$, the query $x_{t} \in \mathcal{X}$ is chosen according to the AIF policy: 
\begin{equation}\label{eqn:policy for learning}
    x_t = \arg\max_{x \in \mathcal{X}} \Big\{ \beta_t I(s; (x,y) \mid \mathcal{D}_{t-1}) - \mathbb{E}_{p(y \mid x, \mathcal{D}_{t-1})}[h_t(y)] \Big\},
\end{equation}
where $I(s; (x,y) \mid \mathcal{D}_{t-1})$ is the conditional mutual information between $s$ and the next observation pair $(x,y)$, and $h_{t}(y)$ is the energy function at time $t$.

Define the per-hypothesis observation gap as
\begin{equation*}
    \Delta_{s}(x_{t}) \coloneq \int_y p(y|x_{t},s)h_{t}(y)-\int_y p(y|x_t,s^*)h_t(y).
\end{equation*}

Suppose the following conditions hold:
\begin{itemize}
    \item[(i)] \textbf{Finite prior entropy.} The prior $q_{0}(s)$ has finite Shannon entropy:
    \begin{equation}
        H_{0} \coloneq H(q_{0}(s)) < \infty;
    \end{equation}
    \item[(ii)] \textbf{Observational distinguishability.} The true parameter $s^{\ast}$ is distinguishable under the observation constraint induced by $h_{t}(y)$:
    \begin{equation}
        \frac{\sum_{s\neq s^*}\Delta_{s}(x_{t})q_{t-1}(s)}{\sum_{s\neq s^*}q_{t-1}(s)} \ge A_{t},
    \end{equation}
    where $A_{t}\ge 0$ quantifies the average discriminative strength of the current observation;
    \item[(iii)] \textbf{Sufficient curiosity.} The curiosity coefficient $\beta_{t}$ satisfies
    \begin{equation}
    \beta_{t} \ge \min_{x \in \mathcal{X}} \frac{\mathbb{E}_{p(y| x, \mathcal{D}_{t-1})}[h_{t}(y)]}{I(s;(x,y)|\mathcal{D}_{t-1})},
    \end{equation}
    ensuring that the value of information is not suppressed by the expected regret penalty.
\end{itemize}
Then, defining $w_{t}\coloneq \sum_{s\neq s^*}q_{t-1}(s)$ as the total posterior probability mass assigned to incorrect hypotheses, it suffices to run for 
\begin{equation}\label{eqn:sample steps}
    T \ge \frac{\bar{\beta}_{T}\; H_{0}}{\underline{A}_{T}\; \epsilon}
\end{equation}
steps to get $\mathbb{E}[w_{T}] \le \epsilon$, where $\underline{A}_{T} \coloneq \min_{t\in[1,T]}A_{t}$ and $\bar{\beta}_{T} \coloneq \max_{t\in[1,T]}\beta_{t}$.
\end{theorem}
\begin{proof}
    See Appendix \ref{appendix:self-consistency}.
\end{proof}

\subsection{Interpretation of Assumptions}

This theorem identifies three complementary requirements for posterior consistency in learning:
\begin{itemize}
    \item \textbf{(i) Finite prior entropy} ensures that the parameter space is \textit{learnable}: the initial uncertainty is bounded, meaning the agent begins with a well-posed inferential task and prior.
    \item \textbf{(ii) Observational distinguishability} guarantees that, under the observation model constrained by $h_{t}(y)$, the true parameter generates statistically distinguishable outcomes from other hypotheses.
    This condition explicitly couples the \textit{informativeness of experiments} with the \textit{expressivity of the energy function}: if $h_{t}(y)$ excessively filters the signal (\eg, by focusing only on narrow outcomes), the system may fail to identify $s^{\ast}$.
    \item \textbf{(iii) Sufficient curiosity} enforces that the epistemic term dominates enough to prevent premature exploitation. 
    When $\beta_{t}$ falls below this threshold, the policy may repeatedly select myopic actions that yield low expected regret but fail to reduce parameter uncertainty, which leads to breaking the consistency guarantee.
\end{itemize}

\subsection{Interpretation of Convergence Rate}

The upper bound $\mathbb{E}[w_{T}] \le \epsilon$ quantifies the expected residual posterior mass on incorrect hypotheses after $T$ queries. 
To reach this level of accuracy, the required number of iterations \eqref{eqn:sample steps}
reveals how \textit{learning efficiency} depends on three key quantities:
\begin{itemize}
    \item $H_{0}$: the \textbf{initial uncertainty} of the model. Larger priors require more evidence to concentrate around the truth.
    \item $\underline{A}_{T}$: the \textbf{minimal discriminative power} of the queries. Experiments that yield greater differences between hypotheses accelerate convergence.
    \item $\bar{\beta}_{T}$: the \textbf{curiosity upper bound}, reflecting how much exploration the agent permits. A higher curiosity encourages richer sampling but may slow convergence toward posterior concentration, as resources are devoted to wide exploration.
\end{itemize}

Thus, consistency emerges not from blind exploration but from \textit{informed curiosity}: $\beta_{t}$ must be large enough to guarantee identifiability, yet not so large as to diffuse effort away from regions of genuine uncertainty reduction.

\subsection{Practical Usability and Implications}

While this theorem assumes idealized conditions (\eg, discrete $\mathcal{S}$ and exact computation of mutual information), it yields valuable \textit{design principles} for practical algorithms:
\begin{itemize}
    \item \textbf{Adaptive curiosity scheduling.} In practice, $\beta_{t}$ can be made \textit{data-dependent}, increasing when the posterior entropy is high or when the expected information gain is small. This can ensure ongoing exploration without manual tuning.
    \item \textbf{Energy function design.} $h_{t}(y)$ can be interpreted as a ``soft constraint'' shaping what the system considers meaningful feedback. Ensuring that $h_{t}(y)$ does not eliminate critical signal differences between hypotheses helps maintain observational distinguishability.
    \item \textbf{Stopping criteria.} The convergence rate about the posterior concentration measure $w_{t}$ provides a direct diagnostic for learning progress, which can be used for stopping criteria or dynamic query allocation in sequential design problems.
\end{itemize}

In summary, Theorem~\ref{thm:self-consistency} establishes that AIF policies satisfying the sufficient curiosity condition yield self-consistent learning under mild assumptions. 
The result formalizes the intuitive idea that curiosity, when bounded but persistent, guarantees truth convergence, which is a cornerstone for the optimization that follows.

\section{No-Regret Optimization}

Having established the conditions for posterior consistency in learning, we now turn to the complementary side of optimization. In contrast to the consistency analysis in the previous section, which concerns the \textit{convergence of beliefs} toward the true model, this section concerns the \textit{quality of decisions} generated by those beliefs. 

We characterize the conditions under which the AIF policy guarantees bounded cumulative regret with respect to an unknown but smooth objective. Importantly, this bound holds even when the true regret function $r(y)$ is unknown or replaced by a heuristic proxy $h_{t}(y)$.

\begin{theorem}[\textbf{Cumulative Regret Bound in AIF}]\label{thm:URB}
Let the unknown function be modeled by a Gaussian Process (GP) prior $f \sim \mathcal{GP}(0, k(x,x'))$, with bounded kernel $k(x,x')\le 1, \forall x, x' \in \mathcal{X}$, and Gaussian likelihood $p(y|f(x))=\mathcal{N}(f(x),\sigma^{2})$.
At each iteration $t$, the query $x_{t} \in \mathcal{X}$ is chosen according to the AIF policy: 
\begin{equation}\label{eqn:policy for optimization}
    x_t = \arg\max_{x \in \mathcal{X}} \Big\{ \beta_t I(s; (x,y) \mid \mathcal{D}_{t-1}) - \mathbb{E}_{p(y \mid x, \mathcal{D}_{t-1})}[h_t(y)] \Big\},
\end{equation}
where $I(s; (x,y) \mid \mathcal{D}_{t-1})$ is the conditional mutual information between $s$ and the next observation pair $(x,y)$, and $h_{t}(y)$ is the energy function at time $t$.

Suppose the following conditions hold:
\begin{itemize}
    \item[(i)] \textbf{Smoothness of task.} The true regret function $r:\mathcal{Y}\rightarrow \mathbb{R}_{\ge 0}$ is Lipschitz continuous with constant $L \ge 0$:
    \begin{equation*}
        |r(y_1)-r(y_2)|\le L|y_1-y_2|, \quad \forall y_{1},y_{2} \in \mathcal{Y};
    \end{equation*}
    \item[(ii)] \textbf{Heuristic alignment.} The discrepancy between the heuristic energy function $h_{t}(y)$ and the true regret $r(y)$ is uniformly bounded at each iteration:
    \begin{equation*}
        |r(y)-h_t(y)|\le B_{t}, \quad \forall y \in \mathcal{Y},
    \end{equation*}
    where $B_{t}$ quantifies the maximal discrepancy of the heuristic to the true regret;
    \item[(iii)] \textbf{Sufficient curiosity.} The curiosity coefficient $\beta_{t}$ satisfies
    \begin{equation}
    \beta_{t} \ge \min_{x \in \mathcal{X}} \frac{\mathbb{E}_{p(y| x, \mathcal{D}_{t-1})}[h_{t}(y)]}{I(s;(x,y)|\mathcal{D}_{t-1})},
    \end{equation}
    ensuring that the value of information is not suppressed by the expected regret penalty.
\end{itemize}
Then, with probability $\ge 1-\delta, \delta \in (0,1)$, the cumulative regret $R_{T}=\sum_{t=1}^{T}r(f(x_{t}))$ satisfies:
\begin{equation}\label{eqn:URB}
    R_{T} \le \bar{\beta}_{T}\rho_{T} + L(\zeta_{T}^{1/2}+{(2/ \pi)}^{1/2}) \sqrt{CT\rho_{T}} + \sum_{t=1}^{T}B_{t},
\end{equation}
where $\bar{\beta}_{T} := \max_{t\in[1,T]}\beta_{t}$, $\zeta_{T}=2\log (m_{T}/\delta)$ with $m_{t}$ satisfies $\sum_{t\ge1}m_{t}^{-1}=1,m_{t}>0$, $\pi$ is the Archimedes' constant, $C=2/\log(1+\sigma^{-2})$ and $\rho_{T}$ is the maximum information gain at most T selected points.
\end{theorem}
\begin{proof}
    See Appendix \ref{appendix:URB}.
\end{proof}

\subsection{Interpretation of Assumptions}

Each assumption in Theorem~\ref{thm:URB} captures a critical property for achieving bounded regret:
\begin{itemize}
    \item \textbf{(i) Smoothness of task} ensures the optimization landscape is well-behaved, \ie, the true regret function is sufficiently smooth so that nearby outcomes yield similar regret values. This smoothness condition allows the GP model to generalize effectively from observed samples.
    \item \textbf{(ii) Heuristic alignment} guarantees that the agent’s internal model of regret $h_t(y)$ remains close to the true objective’s regret. When this alignment condition holds, optimization steps guided by $h_{t}(y)$ remain consistent with the underlying objective.
    \item \textbf{(iii) Sufficient curiosity} enforces that exploration is adequately weighted to avoid premature exploitation, \ie, the agent must not undervalue informative queries. This condition is the same as the curiosity requirement in Theorem~\ref{thm:self-consistency}, highlighting that a minimum exploration pressure is essential both for consistent learning and regret minimization.
\end{itemize}

\subsection{Interpretation of Cumulative Regret Bound}

The bound \eqref{eqn:URB} reveals how cumulative regret depends on three key factors:
\begin{itemize}
    \item $L$: the \textbf{smoothness constant} of the true regret function. A smaller $L$ implies that the regret function landscape is gentler, leading to faster convergence and smaller cumulative regret.
    \item $B_{t}$: the \textbf{bounded difference} between the heuristic and true regret. As the learning process proceeds, if $h_{t}(y)$ is updated to better approximate $r(y)$, $B_{t}$ decreases, leading to a tighter bound.
    \item $\bar{\beta}_{T}$: the \textbf{maximum degree of curiosity}. Larger curiosity encourages broader exploration and yields more information gain, but also increases the exploration term in the bound. Thus, $\beta_{t}$ controls the trade-off between learning and regret accumulation.
\end{itemize}

\subsection{Practical Usability and Implications}

The cumulative regret bound in Theorem~\ref{thm:URB} provides not only theoretical guarantees but also practical \textit{design guidelines} for constructing curiosity-driven optimization systems:
\begin{itemize}
    \item \textbf{Adaptive curiosity scheduling.} In practice, $\beta_t$ should be annealed over time, \ie, starting large to encourage broad exploration and gradually decreasing as the model confidence increases. This strategy ensures convergence while maintaining adaptability to nonstationary environments.
    \item \textbf{Refinement of heuristic estimation.} The cumulative misalignment term $\sum_{t=1}^{T}B_t$ quantifies the cost of misalignment between the heuristic $h_t(y)$ and the true regret $r(y)$. This highlights the importance of iterative model refinement: by updating $h_t(y)$ based on posterior feedback, the agent can continually reduce $B_t$, thereby tightening the bound and improving decision efficiency. Practically, this can be achieved by using learned surrogate models, reward shaping, or consistency constraints derived from task feedback.
    \item \textbf{Task-dependent smoothness and kernel design.} The Lipschitz constant $L$ and kernel function $k(x,x')$ jointly determine how well the GP can generalize across the input space. In real applications, kernel selection should reflect the smoothness of the underlying process. For highly structured or discontinuous objectives, combining multiple kernels or using deep GPs can help maintain bounded regret behavior.
\end{itemize}

Overall, Theorem~\ref{thm:URB} complements Theorem~\ref{thm:self-consistency} by revealing the dual role of curiosity in the AIF paradigm: it enables both posterior consistency in learning and bounded cumulative regret in optimization.
The degree of curiosity thus emerges as a unifying factor that mediates the balance between epistemic (informational) and pragmatic (goal-directed) imperatives, which is a balance that lies at the heart of intelligent behavior.

\section{Experiments for Theorem Validation}

We first validate the derived theoretical results using two intentionally simplified synthetic setups: a discrete hypothesis sandbox for Theorem~\ref{thm:self-consistency} and a 1D GP bandit for Theorem~\ref{thm:URB}. The goal is to \emph{stress-test the mechanisms} implied by the theorems under controlled conditions. These setups (i) directly instantiate the modeling assumptions, (ii) allow us to vary one factor at a time (curiosity, distinguishability, prior uncertainty, heuristic alignment), and (iii) avoid confounders common in real-world tasks (nonstationarity, high-dimensional representation error, and uncontrolled model mismatch), thereby enabling clean causal interpretation.

\subsection{Discrete Sandbox}

To test posterior consistency and the qualitative sample-complexity dependencies in Theorem~\ref{thm:self-consistency}, we consider a controlled discrete environment with finite latent states and actions (details in Appendix~\ref{appx:discrete sandbox}).

We perform one-factor-at-a-time sweeps: (a) \textbf{prior entropy} $H_0$ by varying the prior mass assigned to the true state $q_{0}(s^{\ast})$ in $\{0.25, 0.5, 0.7\}$ at fixed $\beta=2.0$; (b) \textbf{distinguishability}, by modifying the gap structure at fixed $\beta=2.0$ and reporting the empirical pre-convergence average $A_t$ in the plot legend; and (c) \textbf{curiosity} $\beta$ by varying $\beta \in \{2.0, 0.8, 0.05\}$.
We measure posterior contraction using the posterior error mass $w_t = \sum_{s \neq s^*} q_t(s)$, averaged over 5 random seeds.

Figure~\ref{fig:consistency} supports the core qualitative predictions of Theorem~\ref{thm:self-consistency}: (a) \textbf{more accurate prior} yields faster posterior contraction; (b) \textbf{higher discriminative strength} accelerates convergence, while insufficient distinguishability ($<0$) leads to stalled contraction; and (c) \textbf{sufficient curiosity} is necessary for consistent learning.

\begin{figure*}[t!]
    \centering
    \includegraphics[width=0.95\linewidth]{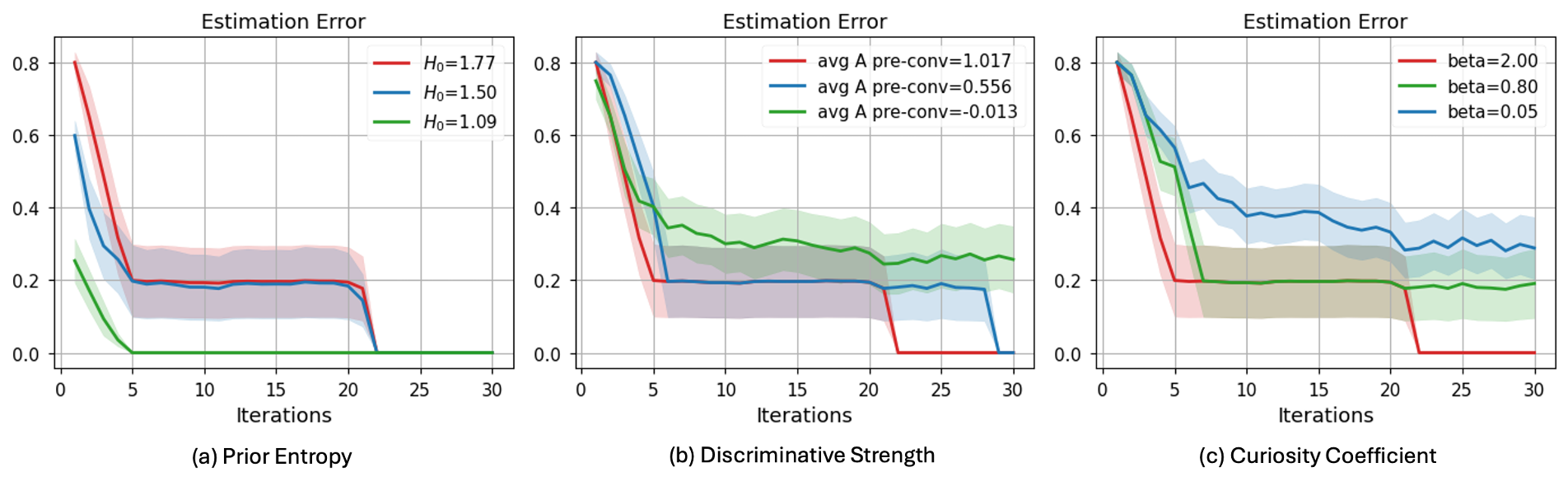}
    \caption{Discrete sandbox to validate Theorem~\ref{thm:self-consistency}. Error bars represent $\pm 0.2$ std over 5 seeds.}
    \label{fig:consistency}
\end{figure*}

\subsection{1D GP Bandit}

To test the regret dependencies in Theorem~\ref{thm:URB}, we consider a minimal 1D GP bandit (details in Appendix~\ref{appx:1d gp}).

We isolate two effects central to the bound: \textit{heuristic alignment} and \textit{curiosity}. Specifically, we run: (a) a \textbf{misalignment} sweep at fixed $\beta=1.0$, comparing an aligned heuristic to biased heuristics with different bias schedules; and (b) a \textbf{curiosity} sweep with aligned heuristics, using $\beta \in \{6.0, 3.0, 1.0, 0.3\}$. We report cumulative regret $R_t$, averaged over 5 seeds.

Figure~\ref{fig:regret} matches the qualitative structure of Theorem~\ref{thm:URB}: (a) \textbf{heuristic error} contributes additively to regret, where larger and longer-lasting bias leads to higher cumulative regret; and (b) when the \textbf{curiosity} is sufficient, smaller $\beta$ (more exploitation) achieves the lowest regret, while large $\beta$ over-emphasizes information gain and incurs additional regret from exploratory queries.

\begin{figure*}[t!]
    \centering
    \includegraphics[width=0.95\linewidth]{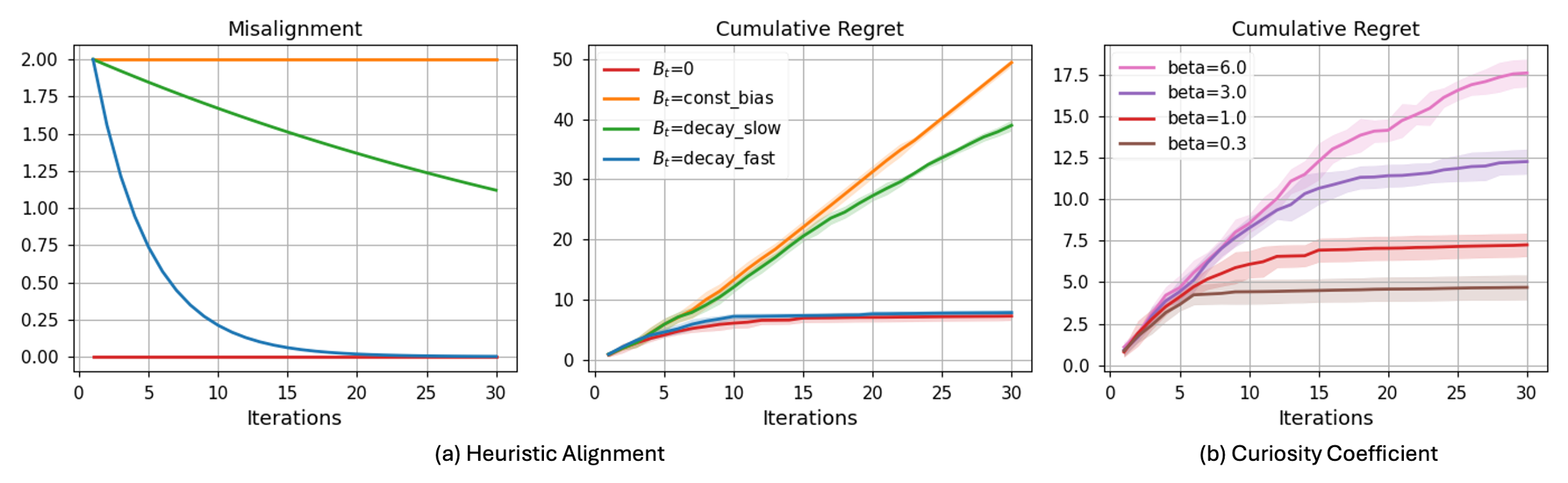}
    \caption{1D GP bandit to validate Theorem~\ref{thm:URB}. Error bars represent $\pm 0.2$ std over 5 seeds.}
    \label{fig:regret}
\end{figure*}

\section{Experiments on Real-World Problems}

Next, we extend the justification of the developed theorems through two categories of real-world problems where learning and optimization are inherently interwined but with different focus on either \textit{epistemic consideration} (\ie, learning), or \textit{pragmatic evaluation} (\ie, optimization).

\subsection{Constrained System Identification}

The objective is to precisely learn unknown system parameters, $\theta$, when valid observations can only be gathered under specific operational constraints, quantified by $C(y)\le0$. 
For instance, when monitoring a chemical plume, the goal is to identify its source ($\theta$) while ensuring that sensor measurements ($y$) do not exceed a saturation threshold ($C(y)$).
This type of task can be found in numerous applications, including environmental monitoring \citep{KonakovicLukovic2020Diversity-guidedEvaluations} and catalyst design \citep{Zhong2020AcceleratedLearning}.

In such cases, it is straightforward to choose the interested state as $s=\theta$, and the energy function as $h(y|\mathcal{D}_{t})=\mathbb{I}(C(y)>0)$. 
This leads to an acquisition function:
\begin{equation*}
\begin{split}
    \alpha(x|\mathcal{D}_{t})=\beta I(\theta;(x,y)|\mathcal{D}_{t}) - \mathbb{E}_{p(y| x, \mathcal{D}_{t})}\left[\mathbb{I}(C(y)>0)\right].
\end{split}
\end{equation*}

We perform experiments on a real-world problem of environmental monitoring in 2d plume fields where the sensors have a saturation threshold $y_{max}$ (\ie, $C(y)=y-y_{max}$). (Detailed settings and hyper-parameter choices in Appendix~\ref{appendix:plume field}).
We consider three types of monitoring tasks: (a) locating the unknown source location; (b) estimating unknown wind direction and strength; and (c) identifying the active sources in fields with multiple sources. 

We investigate how the sufficient curiosity condition depends on task-dependent mutual information by performing ablations over different $\beta$, as illustrated in Figure~\ref{fig:plume_consistency}. From task (a) to (c), the correlation between sensor measurements and latent parameters weakens, so the mutual information term shrinks. In these less informative regimes, a larger $\beta$ is needed to rescale the information-gain term. However, if $\beta$ is too large, the acquisition becomes overly exploratory, which in turn degrades estimation performance.
This yields a simple guideline: start with a moderate $\beta$, increase it only when the mutual information term is consistently small (\eg, due to weak sensor–parameter coupling), and stop or decrease $\beta$ once further increases no longer reduce, or start to increase, the estimation error. In practice, $\beta$ should be just large enough for the curiosity term to matter, but not so large that it induces unnecessary exploration.

\begin{figure*}[t!]
    \centering
    \includegraphics[width=0.95\linewidth]{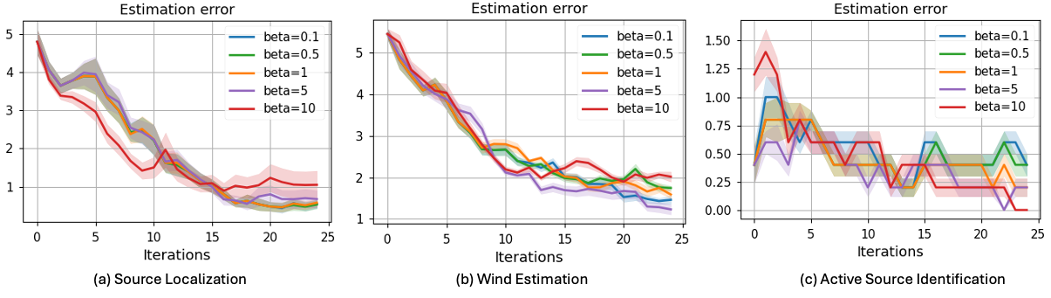}
    \caption{Constrained system identification on environmental monitoring in 2d plume fields. Error bars represent $\pm 0.2$ std over 5 seeds.}
    \label{fig:plume_consistency}
\end{figure*}

\begin{figure*}[t!]
    \centering
    \includegraphics[width=0.95\linewidth]{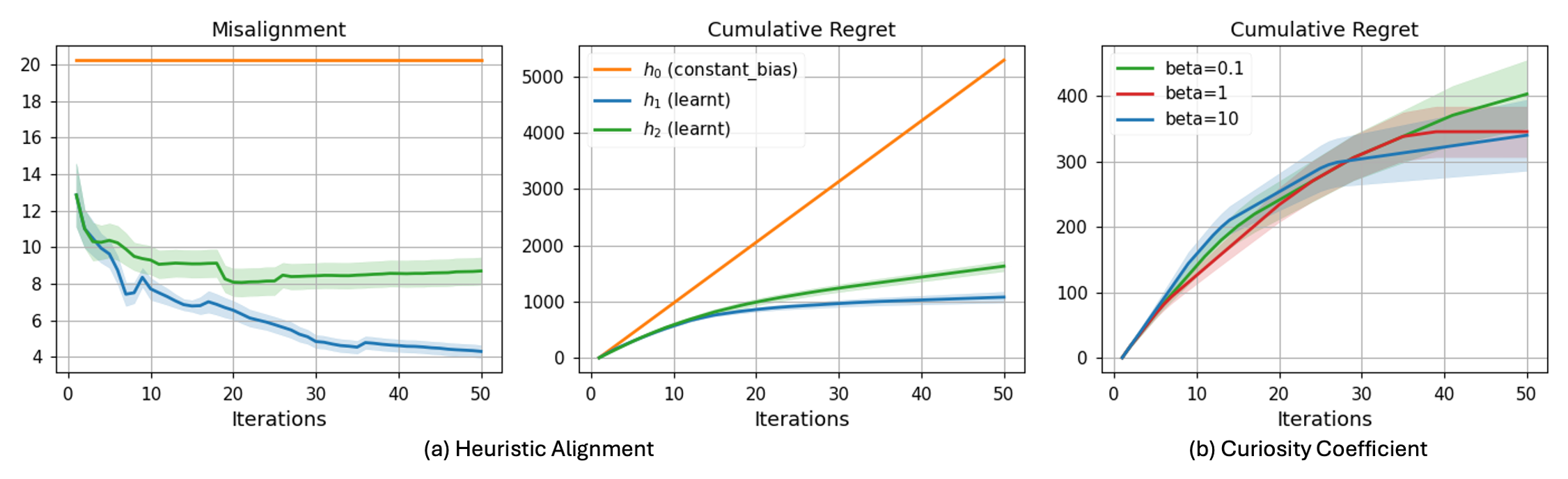}
    \caption{Composite BO on distributed energy resource allocation in power grids. Error bars represent $\pm 0.2$ std over 5 seeds.}
    \label{fig:opf_regret}
\end{figure*}

\subsection{Composite Bayesian Optimization}

A multi-objective optimization problem, where the objectives are weighted by a preference function $g(y)$ that is unknown \textit{a priori} and must be simultaneously learned during the optimization process. Such scenarios commonly arise in simulation-based design \citep{gonzalez2025implementation, coelho2025composite} and A/B testing \citep{Bakshy2018AE:Experimentation}

We expand the interested state into the whole input space, \ie, $s=f_{\mathcal{X}}$, and with non-parametric models like GPs, we can leverage Lemma 3.2 in \cite{AIFBO-application-arxiv} and derive an acquisition function:
\begin{equation} 
\begin{split}
    \alpha(x|\mathcal{D}_{t})=\beta I(f_{x};y|\mathcal{D}_{t}) + \mathbb{E}_{p(y| x, \mathcal{D}_{t})}[h_t(y)],
\end{split}
\end{equation}
where $h_{t}(y)$ is an estimation of $g(y)$ that will be learned online during the optimization process.

We perform experiments on a real-world problem of distributed energy resource allocation in power grids {($x$:40d, $y$:4d) (Detailed settings and hyper-parameter choices in Appendix~\ref{appendix:power grid}).
We investigate the effect of heuristic models and curiosity through two sweeps: (a) fixing $\beta = 1$ with different heuristic models; and (b) fixing a true heuristic model and sweeping $\beta \in \{0.1, 1, 10\}$.

Figure~\ref{fig:opf_regret} further highlights how both the heuristic $h$ and curiosity $\beta$ govern no-regret behavior: (a) cumulative regret stabilizes only when $h$ (and its bias) converges, whereas a constant bias yields approximately linear regret growth; and (b) sufficient curiosity is necessary for regret to converge. Moreover, once $\beta$ is large enough to ensure convergence, increasing $\beta$ incurs larger transient regret before stabilization.

\section{Conclusion and Limitation}

AIF offers a single objective that unifies learning and optimization by trading off epistemic value (information gain) and pragmatic value (expected cost). This paper shows that the reliability of this trade-off is governed by one mechanism: curiosity. When the curiosity coefficient is sufficiently large, the epistemic drive cannot be dominated by the pragmatic term, and EFE minimization becomes provably well-behaved—yielding self-consistent learning (posterior concentration on the true model) together with no-regret optimization (bounded cumulative regret). 
Those theories are further turned into actionable guidance for hybrid learning-optimization design in real-world problems.

These guarantees are sufficient but may be conservative, so deriving tight or adaptive $\beta_t$ schedules remains open. The results also depend on identifiability/regularity assumptions and can degrade under misspecification, nonstationarity, or partial observability. Finally, controlling the pragmatic heuristic alignment in realistic tasks might be challenging, so our mechanism-focused experiments do not replace broad benchmarking.

\section*{Impact Statement}

This work provides a theoretical foundation for active inference as a principled approach to hybrid learning-optimization under uncertainty. By formalizing when EFE minimization yields posterior self-consistency and no-regret behavior, the results can improve the reliability of decision-making systems that must learn online while acting—e.g., robotics, adaptive experimentation, scientific discovery workflows, and safety-critical monitoring—where myopic exploitation or uncontrolled exploration can cause wasted resources or degraded performance. The framework also clarifies how design choices (curiosity level and heuristic alignment) affect learning and regret, offering actionable guidance for practitioners who deploy uncertainty-aware agents.

Potential risks stem from misapplication: setting curiosity too high can lead to unnecessary exploration and inefficiency, while using poorly aligned pragmatic heuristics can bias behavior and increase regret, potentially amplifying errors in downstream decisions. As our guarantees rely on modeling assumptions and heuristic alignment, users should validate calibration and alignment in their target domain and avoid over-interpreting theoretical guarantees as blanket safety assurances. We do not anticipate direct negative societal impacts beyond these standard risks of deploying automated decision-making systems; the primary contribution is methodological and aims to support more transparent and predictable exploration–exploitation trade-offs.

\bibliography{refs}
\bibliographystyle{icml2026}

\newpage
\appendix
\onecolumn

\section{Proof of Theorem \ref{thm:self-consistency}}\label{appendix:self-consistency}
\begin{proof}
According to assumption (iii),
\begin{equation*}
    \max_{x \in \mathcal{X}} \left\{ \beta_{t}I(s;(x,y)|\mathcal{D}_{t-1}) - \mathbb{E}_{p(y| x, \mathcal{D}_{t-1})}[h_{t}(y)] \right\} \ge 0.
\end{equation*}
Then by the decision criteria \eqref{eqn:policy for learning} of $x_{t}$, 
\begin{equation*}
\begin{split}
&\beta_{t}I(s;(x_{t},y_{t})|\mathcal{D}_{t-1}) - \mathbb{E}_{p(y_{t}| x_{t}, \mathcal{D}_{t-1})}[h_{t}(y_{t})] \ge 0.
\end{split}
\end{equation*}
Since $h_{t}(y)\ge 0, \forall y \in \mathcal{Y}$, we have
\begin{equation*}
\begin{split}
&\beta_{t}I(s;(x_{t},y_{t})|\mathcal{D}_{t-1}) - \mathbb{E}_{p(y_{t}| x_{t}, \mathcal{D}_{t-1})}[h_{t}(y_{t})] + \int_y p(y|x_t,s^*)h_t(y) \ge \int_y p(y|x_t,s^*)h_t(y) \ge 0.
\end{split}
\end{equation*}

Define the instant regret of expected observation as
\begin{equation*}
    r_t \coloneq \mathbb{E}_{p(y_{t}| x_{t}, \mathcal{D}_{t-1})}[h_{t}(y_{t})] - \int_y p(y|x_t,s^*)h_t(y),
\end{equation*}
then 
\begin{equation*}
    \beta_{t}I(s;(x_{t},y_{t})|\mathcal{D}_{t-1}) - r_{t} \ge 0.
\end{equation*}

Therefore, 
\begin{equation*}
    I(s;(x_{t},y_{t})|\mathcal{D}_{t-1}) \ge \frac{r_{t}}{\beta_{t}} \ge \frac{r_{t}}{\bar{\beta}_{T}}
\end{equation*}
where $\bar{\beta}_{T} \coloneq \max_{t\in[1,T]}\beta_{t}$.

Let $I_{t}\coloneq I(s;(x_{t},y_{t})|\mathcal{D}_{t-1})$, since
$
    \mathbb{E}[I_t]=\mathbb{E}[H(q_{t-1}(s))-H(q_{t}(s))],
$
then 
\begin{equation*}
    \sum_{t=1}^{T}\mathbb{E}[I_t]=\mathbb{E}[H(q_{0}(s))-H(q_{T}(s))] \le H(q_{0}(s)) < \infty
\end{equation*}

Therefore,
\begin{equation*}
    \sum_{t=1}^{T}\mathbb{E}[r_t] \le \bar{\beta}_{T}\sum_{t=1}^{T}\mathbb{E}[I_t] \le \bar{\beta}_{T} H_{0}.
\end{equation*}

Since 
\begin{equation*}
\begin{split}
    r_t \coloneq & \mathbb{E}_{p(y_{t}| x_{t}, \mathcal{D}_{t-1})}[h_{t}(y_{t})] - \int_y p(y|x_t,s^*)h_t(y) \\
    = &\int_y\sum_s p(y|x_t,s)q_{t-1}(s)h_{t}(y)-\int_y p(y|x_t,s^*)h_t(y) \\
    = &\sum_{s\neq s^*}\left[\int_y p(y|x_{t},s)h_{t}(y)\right]q_{t-1}(s)-\left[\int_y p(y|x_t,s^*)h_t(y)\right](1-q_{t-1}(s^*)) \\
    = &\sum_{s\neq s^*}\left[\int_y p(y|x_{t},s)h_{t}(y)\right]q_{t-1}(s)-\left[\int_y p(y|x_t,s^*)h_t(y)\right]\sum_{s\neq s^*}q_{t-1}(s) \\
    = &\sum_{s\neq s^*}\left[\int_y p(y|x_{t},s)h_{t}(y)-\int_y p(y|x_t,s^*)h_t(y)\right]q_{t-1}(s) \\
    = &\sum_{s\neq s^*}\Delta_{s}(x_{t})q_{t-1}(s),
\end{split}
\end{equation*}
then according to assumption (ii), we have
\begin{equation*}
    r_{t} \ge A_{t} \sum_{s\neq s^*}q_{t-1}(s)=A_{t} w_{t}.
\end{equation*}
Then
\begin{equation*}
    w_{t} \le \frac{r_t}{A_t} \le \frac{r_t}{\underline{A}_{T}},
\end{equation*}
where $\underline{A}_{T} \coloneq \min_{t\in[1,T]}A_{t}$.

Therefore, 
\begin{equation*}
    \sum_{t=1}^{T}\mathbb{E}[w_t] \le \frac{1}{\underline{A}_{T}}\sum_{t=1}^{T}\mathbb{E}[r_t] \le \frac{\bar{\beta}_{T}}{\underline{A}_{T}}\sum_{t=1}^{T}\mathbb{E}[I_t] \le \frac{\bar{\beta}_{T}H_{0}}{\underline{A}_{T}}.
\end{equation*}

Then we get a bound on the time-averaged expected posterior mass: 
\begin{equation*}
    \frac{1}{T} \sum_{t=1}^{T}\mathbb{E}[w_t] \le \frac{\bar{\beta}_{T}H_{0}}{\underline{A}_{T}T}.
\end{equation*}

In particular, there exists at least one time $t \in \left\{1, \dots, T \right\}$ with 
\begin{equation*}
    \mathbb{E}[w_t] \le \frac{\bar{\beta}_{T}H_{0}}{\underline{A}_{T}T}.
\end{equation*}

Equivalently, to get $\mathbb{E}[w_t]\le \epsilon$, it suffices to run for
\begin{equation*}
    T \ge \frac{\bar{\beta}_{T}\; H_{0}}{\underline{A}_{T}\; \epsilon}
\end{equation*}
steps.

\end{proof}

\section{Proof of Theorem \ref{thm:URB}}\label{appendix:URB}

\begin{lemma}[Lemma 5.3 in \cite{Srinivas2009GaussianDesign}]\label{lemma:total information gain}
The information gain for the points selected can be expressed in terms of the predictive variances. If $f_{T}=(f(x_{t}))\in\mathbb{R}^{T}$:
$$I(y_{T};f_{T})=\frac{1}{2}\sum_{t-1}^{T}\log (1+\sigma^{-2}\sigma^{2}_{t-1}(x_{t})).$$
\end{lemma}

\begin{lemma}[Lemma 5.5 in \cite{Srinivas2009GaussianDesign}]\label{lemma:GP bounds}
Pick $\delta \in (0,1)$ and set $\zeta_{t}=2\log (m_{t}/\delta)$, where $\sum_{t\ge1}m_{t}^{-1}=1,m_{t}>0$. Then,
$$|f(x_{t})-\mu_{t-1}(x_{t})|\le \zeta_{t}^{1/2}\sigma_{t-1}(x_{t}) \quad \forall t\ge 1$$
holds with probability $\ge1-\delta$.
\end{lemma}

\begin{lemma}[Lemma 3.2 in \cite{AIFBO-application-arxiv}]\label{lemma:mutual info}
Let $\mathbf{X}\subseteq \mathcal{X}$ be a subset of inputs, and the corresponding function values evaluated at those inputs be denoted as $f_{\mathbf{X}}$. 
Given a historical dataset $\mathcal{D}_{t}$, and new measurements $\mathbf{Y}$ observed at $\mathbf{X}$, the mutual information $$I(f_{\mathcal{X}};(\mathbf{X},\mathbf{Y})|\mathcal{D}_{t})=I(f_{\mathbf{X}};\mathbf{Y}|\mathcal{D}_{t}),$$
for any (finite or infinite) set $\mathcal{X}$.
\end{lemma}

\begin{lemma}\label{lemma:shifted expectation}
If $y\sim \mathcal{N}(\mu,\sigma^2)$, $c\in \mathbb{R}$, then
$$\mathbb{E}[|y-c|]=(\mu-c)[1-2\Phi(\frac{c-\mu}{\sigma})]+\sigma \sqrt{\frac{2}{\pi}}\exp\{-\frac{(c-\mu)^2}{2\sigma^2}\},$$
where $\Phi(\cdot)$ is the standard normal CDF.
\end{lemma}
\begin{proof}
Since $p(y)=\frac{1}{\sqrt{2\pi\sigma^2}}\exp\{-\frac{(y-\mu)^2}{2\sigma^2}\}$, let 
$$z=\frac{y-\mu}{\sigma} \Rightarrow y=\mu+\sigma z, \quad dy=\sigma dz,$$
then $p(y)dy=\phi(z)dz$, where $\phi(z)=\frac{1}{\sqrt{2\pi}}\exp\{-\frac{z^2}{2}\}$ is the standard normal PDF.

Define $a:=\frac{c-\mu}{\sigma}$, then
\begin{equation*}
\begin{split}
    \mathbb{E}[|y-c|]&=\int_{-\infty}^{c}(c-y)p(y)dy+\int_{c}^{\infty}(y-c)p(y)dy \\
    &=\int_{-\infty}^{a}(c-\mu-\sigma z)\phi(z)dz+\int_{a}^{\infty}(\mu-c+\sigma z)\phi(z)dz \\
    &=\int_{-\infty}^{a}(c-\mu)\phi(z)dz-\sigma\int_{-\infty}^{a}z\phi(z)dz+\int_{a}^{\infty}(\mu-c)\phi(z)dz+\sigma\int_{a}^{\infty}z\phi(z)dz \\
    &=(c-\mu)\Phi(a)-\sigma (-\phi(a))+(\mu-c)(1-\Phi(a))+\sigma \phi(a) \\
    &=(\mu-c)(1-2\Phi(a))+2\sigma \phi(a) \\
    &=(\mu-c)[1-2\Phi(\frac{c-\mu}{\sigma})]+\sigma \sqrt{\frac{2}{\pi}}\exp\{-\frac{(c-\mu)^2}{2\sigma^2}\}.
\end{split}
\end{equation*}
\end{proof}

\begin{lemma}\label{lemma:instant regret bound}
Fix $t\ge 1$. If $|f(x_{t})-\mu_{t-1}(x_{t})|\le \zeta_{t}^{1/2}\sigma_{t-1}(x_{t})$, then the instant regret $r(f(x_{t}))$ is bounded by 
\begin{equation*}
    r(f(x_t)) \le \beta_{t}I(f(x_{t});y_{t}|\mathcal{D}_{t-1}) + L(\zeta_{t}^{1/2}+\sqrt{\frac{2}{\pi}})\sigma_{t-1}(x_t) + B_t.
\end{equation*}
\end{lemma}
\begin{proof}
According to assumption (iii),
\begin{equation*}
    \max_{x \in \mathcal{X}} \left\{ \beta_{t}I(s;(x,y)|\mathcal{D}_{t-1}) - \mathbb{E}_{p(y| x, \mathcal{D}_{t-1})}[h_{t}(y)] \right\} \ge 0.
\end{equation*}
Then by the decision criteria \eqref{eqn:policy for optimization} of $x_{t}$, 
\begin{equation*}
\begin{split}
&\beta_{t}I(s;(x_{t},y_{t})|\mathcal{D}_{t-1}) - \mathbb{E}_{p(y_{t}| x_{t}, \mathcal{D}_{t-1})}[h_{t}(y_{t})] \ge 0.
\end{split}
\end{equation*}

Since any interested state $s$ can be fully determined by the whole function $f_{\mathcal{X}}$, we have
\begin{equation*}
\begin{split}
\mathbb{E}_{p(y_{t}| x_{t}, \mathcal{D}_{t-1})}[h_{t}(y_{t})] &\le \beta_{t}I(s;(x_{t},y_{t})|\mathcal{D}_{t-1}) \\
&\le \beta_{t}I(f_{\mathcal{X}};(x_{t},y_{t})|\mathcal{D}_{t-1}) \\
&=\beta_{t}I(f(x_{t});y_{t}|\mathcal{D}_{t-1}),
\end{split}
\end{equation*}
where the last equation holds according to Lemma~\ref{lemma:mutual info}.

Therefore, 
\begin{equation*}
\begin{split}
r(f(x_{t}))&\le r(f(x_{t})) + \beta_{t}I(f(x_{t});y_{t}|\mathcal{D}_{t-1}) - \mathbb{E}_{p(y_{t}| x_{t}, \mathcal{D}_{t-1})}[h_{t}(y_{t})] \\
&=\underbrace{\beta_{t}I(f(x_{t});y_{t}|\mathcal{D}_{t-1})}_{\text{Term 1}} + \underbrace{\mathbb{E}_{p(y_{t}| x_{t}, \mathcal{D}_{t-1})}[r(f(x_{t}))-h_{t}(y_{t})]}_{\text{Term 2}}. \\
\end{split}
\end{equation*}

Term 2 can be decomposed into two parts:
\begin{equation*}
\begin{split}
\text{Term 2} &= \mathbb{E}_{p(y_{t}| x_{t}, \mathcal{D}_{t-1})}[r(f(x_{t}))-r(y_{t})+r(y_{t})-h_{t}(y_{t})] \\
&\le \underbrace{\mathbb{E}_{p(y_{t}| x_{t}, \mathcal{D}_{t-1})}[|r(f(x_{t}))-r(y_{t})|]}_{\text{Term 3}} + \underbrace{\mathbb{E}_{p(y_{t}| x_{t}, \mathcal{D}_{t-1})}[|r(y_{t})-h_{t}(y_{t})|]}_{\text{Term 4}}, \\
\end{split}
\end{equation*}

According to Lemma~\ref{lemma:shifted expectation} and assumption (i) in Theorem~\ref{thm:URB},
\begin{equation*}
\begin{split}
\text{Term 3} \le& \mathbb{E}_{p(y_{t}| x_{t}, \mathcal{D}_{t-1})}[L|f(x_{t})-y_{t}|] \\
=&L\mathbb{E}_{p(y_{t}| x_{t}, \mathcal{D}_{t-1})}[|f(x_{t})-y_{t}|] \\
=&L\{(\mu_{t-1}(x_t)-f(x_t))[1-2\Phi(\frac{f(x_t)-\mu_{t-1}(x_t)}{\sigma_{t-1}(x_t)})]\\
&+\sigma_{t-1}(x_t) \sqrt{\frac{2}{\pi}}\exp\{-\frac{(f(x_t)-\mu_{t-1}(x_t))^2}{2\sigma_{t-1}^{2}(x_t)}\}\},
\end{split}
\end{equation*}
Since $\Phi(\frac{f(x_t)-\mu_{t-1}(x_t)}{\sigma_{t-1}(x_t)}) \in [0,\frac{1}{2}]$ when $\mu_{t-1}(x_t)\ge f(x_t)$, and $\in (\frac{1}{2},1]$ when $\mu_{t-1}(x_t)< f(x_t)$, then
\begin{equation*}
    (\mu_{t-1}(x_t)-f(x_t))[1-2\Phi(\frac{f(x_t)-\mu_{t-1}(x_t)}{\sigma_{t-1}(x_t)})] \le |\mu_{t-1}(x_t) - f(x_t)|.
\end{equation*}
Furthermore, since $\exp\{z\} \in (0,1]$ when $z \le 0$, we have
\begin{equation*}
\begin{split}
\text{Term 3} &\le L[|\mu_{t-1}(x_t) - f(x_t)|+\sqrt{\frac{2}{\pi}}\sigma_{t-1}(x_t)]. \\
\end{split}
\end{equation*}
Since $|f(x_{t})-\mu_{t-1}(x_{t})|\le \zeta_{t}^{1/2}\sigma_{t-1}(x_{t})$, we have 
\begin{equation*}
\begin{split}
\text{Term 3} &\le L[\zeta_{t}^{1/2}\sigma_{t-1}(x_t)+\sqrt{\frac{2}{\pi}}\sigma_{t-1}(x_t)] \\
&= L(\zeta_{t}^{1/2}+\sqrt{\frac{2}{\pi}})\sigma_{t-1}(x_t).
\end{split}
\end{equation*}

According to assumption (ii) in Theorem~\ref{thm:URB}, we have
\begin{equation*}
\begin{split}
\text{Term 4} \le \mathbb{E}_{p(y_{t}| x_{t}, \mathcal{D}_{t-1})}[B_{t}] = B_{t}.
\end{split}
\end{equation*}

Combining all those together, we can obtain
\begin{equation*}
    r(f(x_t)) \le \beta_{t}I(f(x_{t});y_{t}|\mathcal{D}_{t-1}) + L(\zeta_{t}^{1/2}+\sqrt{\frac{2}{\pi}})\sigma_{t-1}(x_t) + B_t.
\end{equation*}

\end{proof}

Now we are ready to prove Theorem~\ref{thm:URB}.
\begin{proof}
Pick $\delta \in (0,1)$ and set $\zeta_{t}$ as in Lemma~\ref{lemma:GP bounds}. Then according to Lemma~\ref{lemma:instant regret bound}, with probability $\ge 1-\delta$, we have 
\begin{equation*}
    r(f(x_t)) \le \beta_{t}I(f(x_{t});y_{t}|\mathcal{D}_{t-1}) + L(\zeta_{t}^{1/2}+\sqrt{\frac{2}{\pi}})\sigma_{t-1}(x_t) + B_t.
\end{equation*}

Note that $\zeta_{t}$ is non-decreasing, then the cumulative regret satisfies:
\begin{equation*}
\begin{split}
R_{T}=&\sum_{t=1}^{T}r(f(x_{t})) \\
\le &\bar{\beta}_{T}\sum_{t=1}^{T}I(f(x_{t});y_{t}|\mathcal{D}_{t-1}) + L(\zeta_{T}^{1/2}+\sqrt{\frac{2}{\pi}}) \sum_{t=1}^{T}\sigma_{t-1}(x_t) + \sum_{t=1}^{T}B_{t} \\
=&\bar{\beta}_{T}I(y_{T};f_{T}) + L(\zeta_{T}^{1/2}+\sqrt{\frac{2}{\pi}}) \sum_{t=1}^{T}\sigma_{t-1}(x_t) + \sum_{t=1}^{T}B_{t},
\end{split}
\end{equation*}
where $\bar{\beta}_{T} \coloneq \max_{t\in[1,T]}\beta_{t}$.

When using GP to model $p(f(x)|\mathcal{D})$, assuming Gaussian likelihood $\mathcal{N}\sim(0,\sigma^{2})$ and the kernel $k(x,x')\le 1, \forall x, x' \in \mathcal{X}$, then $0\le \sigma^{2}(x)\le k(x,x) \le 1$, so we have 
\begin{equation*}
    \sigma_{t-1}^{2}(x_{t}) \le \frac{\log(1+\sigma^{-2}\sigma_{t-1}^{2}(x_{t}))}{\log(1+\sigma^{-2})}.
\end{equation*}

According to Lemma~\ref{lemma:total information gain}, we have 
\begin{equation*}
    \sum_{t=1}^{T}\sigma^{2}_{t-1}(x_t) \le \frac{2}{\log(1+\sigma^{-2})}I(y_{T};f_{T}).
\end{equation*}

By Cauchy-Schwarz inequality, we have 
\begin{equation*}
    \sum_{t=1}^{T}\sigma_{t-1}(x_t) \le \sqrt{T\sum_{t=1}^{T}\sigma^{2}_{t-1}(x_t)} \le \sqrt{T\frac{2I(y_{T};f_{T})}{\log(1+\sigma^{-2})}}.
\end{equation*}

Define $\rho_{T}$ be the maximum information gain at most T selected points. Then we have 
\begin{equation*}
\begin{split}
R_{T} \le& \bar{\beta}_{T}I(y_{T};f_{T}) + L(\zeta_{T}^{1/2}+\sqrt{\frac{2}{\pi}}) \sqrt{T\frac{2I(y_{T};f_{T})}{\log(1+\sigma^{-2})}} + \sum_{t=1}^{T}B_{t} \\
\le& \bar{\beta}_{T}\rho_{T} + L(\zeta_{T}^{1/2}+\sqrt{\frac{2}{\pi}}) \sqrt{\frac{2T\rho_{T}}{\log(1+\sigma^{-2})}} + \sum_{t=1}^{T}B_{t}.
\end{split}
\end{equation*}

\end{proof}

\section{Experimental Details}

\subsection{Discrete Sandbox}\label{appx:discrete sandbox}

The latent state is $s\in\{1,\dots,6\}$, and at each round the agent chooses an action $x\in\{1,\dots,4\}$ and observes a scalar measurement
$$
y_t | (s,x_t)\sim \mathcal{N}(\mu_{s,x_t},\sigma^2),\quad \sigma=0.2,
$$
where the mean table $\mu_{s,x_t}$ is constructed so that a single ``informative'' action yields well-separated likelihoods across states (average gap $\Delta \approx 1.5$), while the remaining actions are nearly indistinguishable ($\Delta \approx 0.05$). 

We define an energy-based heuristic
$$
h_{t}(y)=a \cdot(y-c)^2, \quad c=0,
$$
with $a\in [0.15,0.35]$, which acts as the observation constraint surrogate used in the theorem.

\subsection{1D GP Bandit}\label{appx:1d gp}

We consider 1D inputs $x \in [0,1]$ on a uniform grid of $N=200$ points. The unknown objective is
$$
f(x)=0.6\sin(3\pi x)+0.4\cos(5\pi x)+0.2x.
$$
Observations are noisy:
$$
y_t=f(x_t)+\epsilon_t, \quad \epsilon_t \sim \mathcal{N}(0,\sigma^2), \quad \sigma=0.05.
$$

We compute $I(f; (x,y) \mid \mathcal{D}_{t-1})$ in closed form from the GP predictive variance (mutual information for a Gaussian observation model). 

To isolate the effect of heuristic alignment (Assumption on bounded discrepancy), we use the regret surrogate
$$
h_t(y)=|y-y^{\ast}|+b_{t}(y-y^{\ast}),
$$
where $y^{\ast}$ denotes the best achievable value on the grid (computed from the ground-truth $f$). The aligned setting sets $b_t=0$. Misaligned settings use amplitude $b=2.0$ with schedules:
\begin{itemize}
    \item Constant bias: $b_t=b$.
    \item Slow exponential decay: $b_t=b\exp{(-0.02t)}$.
    \item Fast exponential decay: $b_t=b\exp{(-0.25t)}$.
\end{itemize}

\subsection{Constrained System Identification}\label{appendix:plume field}

\subsubsection{Plume Field Model and Parameters}

We consider the monitoring of a chemical plume field, where multiple plume sources generate plume particles that can be measured by sensors. 
The field function is represented by the rate of hits, defined as the average number of particles per unit time measured by the sensor at a certain location.

The rate of hits for a chemical plume source is given as:
\begin{equation*}
\begin{split}
    R_{\theta}(x)=\frac{R_{s}}{\log \frac{\gamma}{a}}&\exp(-\frac{\langle \theta-x,V \rangle}{2D})K_{0}(\frac{||\theta-x||_{2}}{\gamma}), \\
\end{split}
\end{equation*}
where $\theta$ is the location of the plume source,  $R_{s}$ is the rate at which the plume source releases the plume particles in the environment, $\gamma=\sqrt{D\tau/(1+\frac{||V||^2\tau}{4D})}$ is the average distance traveled by a plume particle in its lifetime, $a$ is the size of the sensor detecting plume particles, $V$ is the average wind velocity, $D$ is the diffusivity of the plume particles, and $K_{0}$ is the Bessel function of zeroth order.

The measurement $y$, \ie, the number of particles measured, is modeled as a Poisson random variable with $R_{\theta}( x)\Delta t$ as the rate parameter, which leads to a likelihood model as
\begin{equation*}
     L_{\theta}(y|x) = \frac{\exp(-R_{\theta}(x)\Delta t) (R_{\theta}(x)\Delta t)^y}{y!},
\end{equation*}
where $\Delta t$ is the time taken to obtain a measurement.

\noindent \textbf{General Field and Sensor Parameters:}
\begin{itemize}
    \item Grid Size: 100 × 100 units.
    \item Sensor Size ($a$): 1.0 units.
    \item Measurement Time ($\Delta t$): 1.0 seconds.
\end{itemize}
    
\noindent \textbf{Source Parameters:}
Eight potential plume sources were defined for the experiments. Their specific physical parameters are detailed as:
\begin{table}[htbp]
\begin{center}
\begin{small}
\begin{sc}
\begin{tabular}{cccccc}
\toprule
Source & $\theta$ & $R_{s}$ & $\gamma$ & $V$ & $D$ \\
\midrule
1 & [20.0, 20.0] & 100.0 & 50.0 & [0.5, 0.5] & 10 \\
2 & [30.0, 80.0] & 100.0 & 60.0 & [-0.3, 0.2] & 15 \\
3 & [45.0, 55.0] & 15.0 & 50.0 & [0.5, 0.5] & 10 \\
4 & [50.0, 50.0] & 18.0 & 30.0 & [-0.3, 0.2] & 15 \\
5 & [55.0, 45.0] & 16.0 & 40.0 & [0.2, -0.4] & 12 \\
6 & [48.0, 52.0] & 17.0 & 35.0 & [0.1, 0.1] & 11 \\
7 & [52.0, 55.0] & 14.0 & 45.0 & [-0.1, -0.1] & 13 \\
8 & [52.0, 52.0] & 18.0 & 40.0 & [0.1, -0.1] & 13 \\
\bottomrule
\end{tabular}
\end{sc}
\end{small}
\end{center}
\vskip -0.15in
\end{table}

\subsubsection{Task-Specific Configurations}

\noindent \textbf{Source Localization.}
\begin{itemize}
    \item Goal: Estimate the unknown location $\theta=[x,y]$ of a single active source.
    \item Ground Truth: Source 1 was configured as the single active source.
    \item Maximum Sensor Threshold ($y_{max}$): 60.0 hits/second.
    \item Hypothesis Space: A discrete grid of potential locations spanning [0,100]×[0,100] with a resolution of 5 units, resulting in a 20 × 20 grid of 400 hypotheses.
\end{itemize}

\noindent \textbf{Wind Estimation.}
\begin{itemize}
    \item Goal: Estimate the unknown wind vector $V=[v_x,v_y]$ for a single source with a known location.
    \item Ground Truth: Source 2 was used, with its true wind vector set to [-0.3, 0.2].
    \item Maximum Sensor Threshold ($y_{max}$): 60.0 hits/second.
    \item Hypothesis Space: A discrete grid of potential wind vectors spanning [-1,1]×[-1,1] with a resolution of 0.1 units, resulting in a 20 × 20 grid of 400 hypotheses.
\end{itemize}

\noindent \textbf{Active Source Identification.}
\begin{itemize}
    \item Goal: Identify the subset of active sources from a set of six potential sources with known locations.
    \item Ground Truth: Sources 3-8 were used, with Sources 3, 5, 6, and 8 were set as active.
    \item Maximum Sensor Threshold ($y_{max}$): 30.0 hits/second.
    \item Hypothesis Space: The set of all possible on/off combinations for the six sources. This is a discrete space with $2^6=64$ unique hypotheses.
\end{itemize}

\subsection{Composite Bayesian Optimization}\label{appendix:power grid}

\subsubsection{Heuristic Estimation through Preference Learning}
As in \citep{Lin2022PreferenceOutcomes}, we estimate the unknown preference function $g(y)$ by asking a decision-maker (DM) to express preferences over pairs of outcomes $(y_1, y_2)$. Let $z(y_1, y_2) \in \{1, 2\}$ indicate whether the DM preferred the first or second outcome offered.
Following \citep{Chu2005PreferenceProcesses}, we assume that the DM’s responses are distributed according to a \textit{probit} likelihood:
$
    L(z(y_1, y_2)=1|g(y_1),g(y_2))=\Phi(\frac{g(y_1)-g(y_2)}{\sqrt{2}\lambda}),
$
where $\lambda$ is a hyperparameter, and $\Phi$ is the standard normal CDF.

Following \cite{AIFBO-application-arxiv}, taking the joint exploration and exploitation of the preference function $g(y)$ into consideration, we can choose $h(y|\mathcal{D}_{t})$ as the EFE of the preference model, \ie, $h(y|\mathcal{D}_{t})=\text{EFE}(g(y))$, which naturally extends the acquisition function into a nested structure:
\begin{equation}\label{eqn:AF nested}
    \begin{split}
        \alpha(x|\mathcal{D}_{t}) = \beta I(f_{x};y|\mathcal{D}_{t}) +\mathbb{E}_{p(y| x, \mathcal{D}_{t})}[\gamma I(g_{y};z|\mathcal{D}_{t})+\mathbb{E}_{p(g_{y}|y,\mathcal{D}_{t})}g_y]],
    \end{split}
\end{equation}
where $x=[x_{1},x_{2}]$ are two jointly evaluated candidates, and $\gamma, \beta \ge 0$.

\subsubsection{Task-Specific Configurations}

\begin{table}[htbp]\label{tab:opf}
\begin{center}
\begin{small}
\begin{tabular}{l p{10cm}}
\toprule
Performance Metrics & Definition  \\
\midrule
Voltage Fairness & Measures the variance in bus voltages across the network; lower variance implies more equitable voltage delivery. \\
\hline
Total Cost & Combines capital expenditures for DER installation and operational costs related to reactive power support. \\
\hline
Priority Area Coverage & Quantifies the share of power delivered to high-priority buses, such as rural or underserved regions. \\
\hline
Resilience & Assesses the percentage of time that all bus voltages remain within safe operating limits under perturbations (\eg, load uncertainty or line outages). \\
\bottomrule
\end{tabular}
\end{small}
\end{center}
\end{table}

\textbf{Energy Resource Allocation.}
\begin{itemize}
    \item Goal: Identify deployment strategies for Distributed Energy Resources (DERs) in Optimal Power Flow (OPF) that align with implicit ethical preferences across multiple performance dimensions detailed in the following table. 
    \item Testbed: IEEE 5-bus network in pandapower library.
\end{itemize}

\textbf{Regret and Heuristic models.}
\begin{itemize}
    \item Ground Truth Preference: $g(y)=a^{\intercal}y$, where $a=[1,-1,2,1]$
    \item Constant Bias $h_{0}$: $h_{0}(y)=-a_{0}^{\intercal}y$, where $a_{0}=[1,-1,10,1]$
    \item Learnt $h_{1}$: Query according to \eqref{eqn:AF nested} with $\gamma = 1$
    \item Learnt $h_{2}$: Query according to \eqref{eqn:AF nested} with $\gamma = 10$
\end{itemize}

\end{document}